\documentclass{article}
 \usepackage[preprint]{neurips_2026}

\usepackage{amsmath,amsfonts,bm}

\def\eqref#1{equation~\ref{#1}}
\def\1{\bm{1}}

\def\vh{{\bm{h}}}

\def\vz{{\bm{z}}}

\DeclareMathAlphabet{\mathsfit}{\encodingdefault}{\sfdefault}{m}{sl}
\SetMathAlphabet{\mathsfit}{bold}{\encodingdefault}{\sfdefault}{bx}{n}

\def\gA{{\mathcal{A}}}

\def\gE{{\mathcal{E}}}

\def\gG{{\mathcal{G}}}
\def\gH{{\mathcal{H}}}

\def\gQ{{\mathcal{Q}}}
\def\gR{{\mathcal{R}}}

\def\gV{{\mathcal{V}}}

\def\gX{{\mathcal{X}}}

\newcommand{\R}{\mathbb{R}}

\usepackage[utf8]{inputenc} % allow utf-8 input
\usepackage[T1]{fontenc}    % use 8-bit T1 fonts
\usepackage{hyperref}       % hyperlinks
\usepackage{url}            % simple URL typesetting
\usepackage{booktabs}       % professional-quality tables
\usepackage{amsfonts}       % blackboard math symbols
\usepackage{nicefrac}       % compact symbols for 1/2, etc.
\usepackage{microtype}      % microtypography
\usepackage{xcolor}         % colors
\usepackage{amsthm}
\usepackage{algorithm}
\usepackage{algpseudocode}
\usepackage[most]{tcolorbox} % boxed callouts

\usepackage{amssymb}
\usepackage{dsfont}

\usepackage{bbold}

\newcommand{\Pa}{\operatorname{Pa}}

\newcommand{\argsort}{\operatorname{argsort}}

\newtheorem{proposition}{Proposition}

\usepackage[table]{xcolor}
\usepackage{siunitx}

\usepackage{wrapfig}
\usepackage{multirow}
\usepackage{xspace}

\definecolor{tokhigh}{HTML}{6BAED6} % low token usage (better): lighter blue
\definecolor{toklow}{HTML}{F4A582}  % high token usage (worse): soft coral

\definecolor{lightred}{rgb}{1,0.8,0.8} 
\definecolor{lightgreen}{rgb}{0.8,1,0.8}
\definecolor{lightblue}{rgb}{0.88,0.96,1}
\definecolor{lightgray}{rgb}{0.9,0.9,0.9}

\newcommand{\method}{\textsc{Nexa}\xspace}
\newcommand{\tokcell}[2]{%
    \cellcolor{tokhigh!#2!toklow}\num[group-separator={,}]{#1}%
}

\title{Response-Conditioned Parallel-to-Sequential Orchestration for Multi-Agent Systems}

\author{%
  Nurbek Tastan\textsuperscript{1,2} \quad Alex Iacob\textsuperscript{2,3} \quad Lorenzo Sani\textsuperscript{2,3} \quad Meghdad Kurmanji\textsuperscript{2} \\ \textbf{Nicholas D. Lane\textsuperscript{2,3}} \quad \textbf{Samuel Horv\'{a}th\textsuperscript{1}} \quad \textbf{Karthik Nandakumar\textsuperscript{1,4}} \\ 
  \textsuperscript{1}MBZUAI, UAE \quad \textsuperscript{2}University of Cambridge, UK \\ \textsuperscript{3}Flower Labs, UK \quad \textsuperscript{4}Michigan State University, USA \\ 
}

\begin{document}

\maketitle

\begin{abstract}
    Multi-agent systems can solve complex tasks through collaboration between multiple Large Language Model agents. Existing collaboration frameworks typically operate in either a parallel or a sequential mode. In the parallel mode, agents respond independently to queries followed by aggregation of responses. In contrast, sequential systems allow agents to communicate via a directed topology and refine one another step by step. However, both modes are inadequate for achieving the desired objectives of minimizing communication and latency while simultaneously maximizing the accuracy of the final response. In this work, we introduce a hybrid paradigm called \method, a trainable response-conditioned policy that bridges the gap between the two modes. \method begins with a parallel execution stage, embeds the resulting responses into a shared semantic space, and then predicts a sparse directed acyclic communication graph. If the graph is empty, the system remains purely parallel; if it is non-empty, the system performs one sequential message propagation. The policy is a lightweight transformer model, and the method avoids the need for external LLM judges or reward models, as well as hand-crafted test-time topology search. We formalize this hybrid execution problem, show that the resulting graph is acyclic by construction, and that the framework strictly subsumes pure parallel execution, and present a training procedure based on policy-gradient optimization. Results demonstrate that the response-conditioned policy learned by \method under one setting can be reused when the number of agents, the task, or the underlying agent changes, thus emphasizing the generalizability of the learned communication policy. 
\end{abstract}

\vspace{-0.5em}
\section{Introduction}
\vspace{-0.5em}

Large language models (LLMs) have become increasingly capable at reasoning, coding, planning, and dialogue, yet a single model still suffers from stochastic failures, brittle long-horizon reasoning, and occasional hallucinations. Multi-agent systems aim to address these weaknesses by distributing problem solving across multiple agents whose outputs can complement, critique, or refine one another. The central question of such systems is \textbf{how that collaboration should be orchestrated.} 

Existing LLM-based multi-agent systems largely fall into two categories. In \textbf{parallel} systems, agents answer independently, and their outputs are combined by majority voting, self-consistency, or a learned aggregation rule \citep{wang2023selfconsistency, jiang2023llmblender}. In \textbf{sequential} systems, agents are arranged in a communication topology, often a chain, tree, or a more general graph, and information is propagated step by step \citep{zhuge24gptswarm, qian2025scaling}. Parallel systems are simple and scalable, but they are computationally expensive, token-intensive, and often redundant, requiring multiple rounds of parallel message propagation while still being unable to exploit targeted communication when one draft could help repair another. Sequential systems can support error correction and information flow, but they require a topology and therefore inherit the burden of deciding who should communicate with whom. Prior work has explored fixed topologies, policy-gradient optimization over edges, graph generators conditioned on tasks or roles, and judge-based routing, each adding substantial token, compute, optimization, coordination overhead, or reducing transferability across settings~\citep{qian2025scaling, zhuge24gptswarm, zhang2025gdesigner}.

These two paradigms are often treated as separate design choices. A system is either built as a parallel ensemble or as a sequential graph-based collaboration mechanism. Yet, this distinction is too rigid. In many realistic settings, the right approach is not to commit in advance to a single paradigm, but to start in parallel and then decide, based on the agent's actual outputs, whether sequential propagation is necessary. If the initial responses already contain strong agreement and sufficient information, additional communication may be unnecessary. If they disagree in informative ways, or if useful signals are scattered across agents, then structured propagation may help. This suggests that the real problem is not ``parallel or sequential'' in the abstract, but rather: 
\begin{tcolorbox}[
  colback=blue!10,
  colframe=black,
  boxrule=1pt,
  arc=4pt,
  left=3mm,right=3mm,top=1mm,bottom=1mm,
  halign=center
]
Given the current pool of agent responses, should the system remain in the parallel regime, or should it instantiate a communication graph and perform sequential refinement?
\end{tcolorbox}
To answer this question, we introduce \textbf{\method} (from ``nexus'', a connection or link), a trainable policy for \emph{communication graph prediction} in multi-agent LLM systems. \method begins with a parallel draft stage in which all agents answer independently. The resulting response pool is embedded into a shared semantic space, producing a compact representation of the current response state of the team. A lightweight transformer-based policy then predicts a sparse directed acyclic graph (DAG). If the graph is empty, the system remains in the parallel regime and returns the parallel aggregate. If the graph is non-empty, the system executes one sequential consolidation pass in which selected agents update their responses using information from upstream~nodes. 

This formulation deliberately treats parallel and sequential execution not as mutually exclusive system designs, but as two outcomes of the same learned policy. In this sense, the central contribution of \method is not merely graph prediction. It is a mechanism for \emph{bridging the gap between parallel execution and sequential execution} by using parallel drafts to decide whether structured propagation is needed and, if so, how it should proceed. 

A second principle of the method is simplicity. We do \emph{not} learn the topological order. Instead, we induce the order from agent contributions, retaining the most stable organizing principle of response-conditioned communication. The policy learns only the communication edges. We score the candidate communication edges using the affinity matrix formed from transformer-contextualized response representations. This makes the policy lightweight and keeps the graph decoder tightly coupled to the semantic interactions encoded by the backbone. 

\method is also designed to be agnostic to superficial configuration details. The policy consumes semantic representations of agent outputs rather than role labels, agent identities, or model-family indicators. As a consequence, the planner is structurally insensitive to which agent is called ``Programmer'' or ``Assistant''; what matters is what the agents actually say. This does \emph{not} by itself guarantee transfer across all tasks or backbones, and we explicitly treat that as an empirical question. But it does mean that the policy class is not intrinsically tied to a fixed role inventory or a single team structure. 

The paper makes four contributions. First, it formalizes a hybrid decision problem in which a learned communication graph determines whether a multi-agent system remains parallel or enters a sequential propagation regime. Second, it proposes a contribution-ordered, attention-based graph policy that predicts only the communication edges, keeping the controller simple and acyclic by construction. Third, it integrates the key theoretical properties of the method directly into the formulation: DAG validity, hybrid subsumption, and permutation-based identity agnosticism. Fourth, it empirically evaluates \method across reasoning and programming tasks, showing improved accuracy-cost tradeoffs, sparse communication behavior, and transfer across agent counts, tasks, model scales, and generations.

\vspace{-1em}
\section{Problem Formulation and Preliminaries}
\vspace{-.75em}

Let $\mathcal{A} \in \{\mathcal{A}_1, \ldots, \mathcal{A}_N\}$ be a set of $N$ agents, and let $\mathcal{Q}$ be a user query. Each agent may differ in prompt, role, or backbone model, but the communication policy introduced in this work does not rely on these identities explicitly. Instead, it operates on the semantic content of the agents' responses. 

Given the query, each agent independently produces an initial response 
\begin{equation}
    \gR_n^{(0)} = \gA_n(\gQ), \qquad n \in \{1, 2, \ldots, N\}. 
\end{equation} 
The first phase is fully parallel and produces a draft response set $\gR^{(0)} = \{\gR_1^{(0)}, \ldots, \gR_N^{(0)}\}$. 

The purpose of this is twofold. First, it provides diverse candidate solutions to the query. Second, and more importantly for our setting, it exposes the \textbf{current response state} of the multi-agent system. Since LLM outputs are inherently stochastic, this realized state is more informative for downstream coordination than static task labels or role descriptions. This response-conditioned perspective is central to the present work and follows the same foundational motivation that underlies SelfOrg~\citep{tastan2026stochastic}. 

To reason about relations among agent outputs, we map each response into a shared semantic embedding space using a fixed lightweight encoder $f$ (all-MiniLM-L6-v2~\citep{reimers2019sentencebert}): 
% \begin{equation}
    $r_n = f(\gR_n^{(0)}) \in \R^d.$ 
% \end{equation} 

Following SelfOrg~\citep{tastan2026stochastic}, we define the average response embedding 
% \begin{equation}
    $r_{\text{avg}} = \frac{1}{N} \sum_{n=1}^N r_n$
% \end{equation}
and contribution scores 
% \begin{equation}
    $\psi_n = \cos(r_n, r_{\text{avg}}).$
% \end{equation}
%
SelfOrg motivates $\psi_n$ as a linear-time approximation to a Shapley-style contribution value \citep{shapley1953value} and shows that, under suitable separation conditions, ranking by $\psi_n$ preserves the normalized Shapley ordering. This is precisely why \method uses contribution to define the topological ordering of the edges. 

The orchestration problem is to predict a directed communication graph 
% \begin{equation}
    $\gG = (\gV, \gE, \pi),$ 
% \end{equation}
where $V=\{1, \ldots, N\}, E \subseteq V \times V,$ and $\pi$ is an order over the nodes. If the graph is empty ($\gE = \varnothing$), the system stays in the parallel regime and outputs an aggregate of the initial drafts. If it is non-empty ($\gE \neq \varnothing$), the graph induces a sequential propagation step. 

For each node $n$, define its parent set 
% \begin{equation}
    $\Pa(n) = \{m: (m \to n) \in \gE\}.$ 
% \end{equation}
%
Then the updated response is 
\begin{equation}
    \gR_n^{(1)} = 
    \begin{cases}
        \gA_n (\gQ, \{\gR_m^{(\star)}: m \in \Pa(n)\}), & \Pa(n) \neq \varnothing, \\ 
        \gR_n^{(0)}, & \Pa(n) = \varnothing, 
    \end{cases}
\end{equation}
where $\gR_m^{(\star)}$ denotes the most recent available parent response under the topological execution order. 

The final answer is selected from the resulting response pool using a judge-free aggregation rule. Let $\{\vz_n\}$ be the final response embeddings and $\{w_n\}$ their contribution weights. We compute 
\begin{equation}
    z_{\text{centroid}} = \frac{\sum_{n=1}^N w_n \vz_n}{\sum_{n=1}^N w_n}, \qquad n^{\star} = \arg\max_n \cos(\vz_n, \vz_{\text{centroid}}). 
\end{equation}
and return the corresponding response. 

The learning objective is to maximize final task correctness. Given ground truth $y$ and final prediction $\hat{y}_{\gG}$ under graph $\gG$, the reward is 
\begin{equation}
    R(\gG) = \mathds{1}\left[ \operatorname{Eval}(\hat{y}_{\gG}, y) = 1 \right]. 
\end{equation}
The policy, therefore, learns to predict a communication graph that determines whether the initial parallel responses should remain as they are or be further refined through structured propagation. 

\vspace{-1em}
\section{Methodology} 
\vspace{-0.5em}
\label{t-selforg: methodology}

\subsection{System Overview} 

\method consists of five stages. First, all agents produce draft responses in parallel. Second, those responses are embedded into a shared semantic space. Third, a response-conditioned transformer policy predicts a sparse communication graph. Fourth, if the graph is non-empty, the corresponding destination nodes are updated sequentially. Fifth, the final answer is selected from the resulting response pool by weighted-centroid-based aggregation. 

This design has a central conceptual consequence: parallel execution is not discarded when sequential communication is introduced. Instead, the parallel draft stage becomes the \emph{source of evidence} that determines whether the system should remain in the parallel regime or transition into a sequential propagation regime.

\subsection{Contribution-Defined Order and DAG Validity}

We set the topological order as 
% \begin{equation}
    $\pi = \argsort(\psi_1,\dots,\psi_N; \psi_k \geq \psi_{k+1}, \forall k \in [N]).$ 
% \end{equation}

In other words, higher-contribution agents are always placed earlier in the communication order. The feasible edge set is therefore restricted to 
\begin{equation}
    \gE_{\pi} = \{(m,n) : \pi^{-1}(m) < \pi^{-1}(n)\}, 
\end{equation}
so that communication is only allowed to move forward under the contribution order. 

\begin{proposition}[Acyclicity by construction]
For any edge set $\gE \subseteq \gE_{\pi}$, the graph $\gG=(\gV,\gE,\pi)$ is a directed acyclic graph. 
\end{proposition}
\vspace{-1em}
\begin{proof}
Assume for contradiction that $\gG$ contains a directed cycle 
\begin{equation}
    v_1 \to v_2 \to \cdots \to v_K \to v_1.
\end{equation}

Because every edge must go forward under $\pi$, we must simultaneously have
\begin{equation}
\pi^{-1}(v_1) < \pi^{-1}(v_2) < \cdots < \pi^{-1}(v_K) < \pi^{-1}(v_1),
\end{equation}
which is impossible. Hence, no directed cycle can exist. 
\end{proof}
\vspace{-1em}
This parameterization is simpler than detecting and repairing cycles after graph prediction because DAG validity is built directly into the action space of the policy. 

\vspace{-0.5em}
\subsection{Response-Conditioned Graph Policy} 
\vspace{-0.5em}

The graph policy consumes only the current response set, not agent identities, role labels, or model-family indicators. Let
\begin{equation}
    \gX = [r_1,\dots,r_N]^{\top} \in \mathbb{R}^{N \times d}. 
\end{equation}
A transformer encoder~\citep{vaswani2017attention} $\operatorname{Enc}_{\theta}$ maps the response set to contextualized node states
\begin{equation}
    \gH = \operatorname{Enc}_{\theta}(X) = [\vh_1, \dots, \vh_N]^{\top},
    \qquad
    \vh_n \in \mathbb{R}^{d_h}.
\end{equation}

Because the encoder operates on the response embeddings without identity-specific tokens, the policy is permutation-equivariant over the agent dimension. 

\begin{proposition}[Permutation-based identity agnosticism]
Assume that the response encoder $f$ is applied independently to each response and that $\operatorname{Enc}_{\theta}$ is permutation-equivariant. Then, for any permutation matrix $P$,
\begin{equation}
    \operatorname{Enc}_{\theta}(P\gX) = P \operatorname{Enc}_{\theta}(\gX), 
\end{equation}
 the induced graph distribution is equivariant to any relabeling of agents. 
\end{proposition}
\vspace{-1.5em}
\begin{proof}
Since $\operatorname{Enc}_{\theta}$ (transformer) without positional encodings is permutation-equivariant, permuting agent indices permutes $\gX$ and thus $\operatorname{Enc}_\theta(\gX)$; the remaining steps (cosine-to-mean scoring, ordering, and edge construction) are permutation-consistent, so the graph distribution is equivariant. 
\end{proof}
\vspace{-0.5em}

We then predict communication edges directly from the globally contextualized hidden states. Concretely, we form a response-response score matrix from the contextualized states: 
\vspace{-0.5em}
\begin{equation}
    \Lambda = \gH \gH^{\top}. 
\end{equation}
\vspace{-0.2em}

Here, $\Lambda$ is not the final adjacency matrix; it provides edge logits that are passed through a sigmoid and then sampled to obtain the communication graph. This construction is deliberate. The hidden states in $\gH$ are already globally contextualized, so the resulting edge logits are informed by the entire response set rather than by isolated response pairs. In this way, the communication graph is read directly from the shared semantic structure induced by the encoder. 

\subsection{Response Propagation and Aggregation}

\method does not require a separate stop network or node-activation network. The graph itself determines both whether sequential communication occurs and which nodes are updated. A node is updated if and only if it has at least one incoming edge: 
\begin{equation}
    u_n = \1 \left[\sum_{m=1}^{N} \1[(m \to n) \in \gE] > 0\right]. 
\end{equation}
If $\gE = \varnothing$, then $u_n = 0$ for all nodes, no additional calls are made, and the system returns the parallel aggregate. If $\gE \neq \varnothing$, the graph induces one sequential consolidation pass. 
\begin{proposition}[Hybrid subsumption]
    The policy class of \method strictly subsumes the pure parallel regime. 
\end{proposition}
\vspace{-1.2em}
\begin{proof}
    The empty graph $\gE=\varnothing$ is always attainable (all edge probabilities zero/small), in which case no updates occur, and the method reduces to pure parallel execution with aggregation. Any $\gE\neq\varnothing$ induces at least one sequential update, so the policy class strictly contains the parallel regime. 
\end{proof}
\vspace{-0.8em}
When $\gE \neq \varnothing$, sequential propagation follows the contribution order $\pi$. For node $n$, define 
\begin{equation}
    \Pa(n) = \{m: (m \to n) \in \gE\}.
\end{equation}
The updated response is then
\begin{equation}
    \gR_n^{(1)} =
    \begin{cases}
        \gA_n\!\left(\gQ, \{\gR_m^{(\star)} : m \in \Pa(n)\}\right), & \Pa(n) \neq \varnothing, \\
        \gR_n^{(0)}, & \Pa(n) = \varnothing,
    \end{cases}
\end{equation}
where $\gR_m^{(\star)}$ denotes the most recent available parent response under the topological execution order. Because all edges go forward under $\pi$, each parent response is available when a destination node is~updated. 

After either staying in the parallel regime or completing one propagation pass, \method selects the final answer without using an external judge. Let $\tilde{\gR}_n$ denote the final candidate response for agent $n$ and let
\begin{equation}
    z_n = f(\tilde{\gR}_n),
    \qquad
    z_{\text{avg}} = \frac{1}{N} \sum_{n=1}^{N} z_n,
    \qquad
    w_n = \cos(z_n, z_{\text{avg}}).
\end{equation}
We then compute the contribution-weighted centroid
% \begin{equation}
    $z_{\text{centroid}} = \frac{\sum_{n=1}^{N} w_n z_n}{\sum_{n=1}^{N} w_n}$
% \end{equation}
and select
\begin{equation}
    n^{\star} = \arg\max_n \cos(z_n, z_{\text{centroid}}),
    \qquad
    \hat{y} = \tilde{\gR}_{n^{\star}}.
\end{equation}
This aggregation rule directly inherits the response-conditioned, judge-free philosophy of SelfOrg~\citep{tastan2026stochastic}. 
\subsection{Training Objective} 

The deployment objective is the final task correctness. For a labeled example $(\gQ,y)$, let $\hat{y}_{\gG}$ denote the final output under graph $\gG$. In the current implementation, correctness is checked with the same verifier used in evaluation, instantiated as an \texttt{xVerify}-based binary reward \citep{chen2025xverifyefficientanswerverifier}. We therefore define the task reward 
\begin{equation}
    R_{\mathrm{task}}(\gG) = \mathds{1}\left[  \operatorname{Eval}(\hat{y}_{\gG}, y) = 1 \right]. 
\end{equation}
Because the order $\pi$ is fixed by the contribution scores, the graph log-probability decomposes over feasible forward edges:
\begin{equation}
    \log p_{\theta}(\gE \mid \gX,\pi) = \sum_{(m,n)\in \gE_{\pi}} \Big( e_{m \to n}\log p_{m \to n} + (1-e_{m \to n})\log(1-p_{m \to n}) \Big). 
\end{equation}
The algorithm also applies an explicit sparsity penalty to the sampled graph reward in the same spirit as topology-economical methods \citep{zhang2025agentprune}. Let
% \begin{equation}
    $M = |\gE_{\pi}| = \frac{N(N-1)}{2}$
% \end{equation}
be the number of feasible forward edges under the contribution-defined order. For a sampled graph $\gG$, we define the sparsity-regularized reward
\begin{equation}
    R_{\mathrm{sp}}(\gG)
    =
    R_{\mathrm{task}}(\gG)
    -
    \lambda_{\mathrm{sp}} \frac{|\gE|}{M},
\end{equation}
where $\lambda_{\mathrm{sp}} \ge 0$ controls how strongly dense communication graphs are penalized.

We train \method with REINFORCE and a batch-mean baseline. For a mini-batch of sampled graphs $\{\gG^{(i)}\}_{i=1}^{B}$, we set
\begin{equation}
    b = \frac{1}{B}\sum_{i=1}^{B} R_{\mathrm{sp}}(\gG^{(i)}),
    \qquad
    A^{(i)} = R_{\mathrm{sp}}(\gG^{(i)}) - b.
\end{equation}
The policy-gradient term is
\vspace{-1em}
\begin{equation}
    \mathcal{L} = -\frac{1}{B}\sum_{i=1}^{B} A^{(i)} \log p_{\theta}(\gE^{(i)} \mid \gX^{(i)},\pi^{(i)}). 
\end{equation}
Finally, we obtain the final optimization goal: it is REINFORCE with batch-mean advantage, while sparsity is enforced through the edge-count penalty in the reward. 

The full procedure is summarized in Algorithm~\ref{alg:tselforg}. 
\vspace{-.5em}
\section{Experiments}
\label{tselforg: experiments}

The empirical study is designed to demonstrate the generalizability of the learned communication policy. Rather than only reporting in-domain performance for the training configuration, we evaluate whether a response-conditioned policy learned in one setting can be reused when the number of agents, the task, or the underlying agent changes, thus emphasizing training efficiency. 
\vspace{-0.5em}
\subsection{Experimental Setup}
\vspace{-0.5em}
The base training setting uses Qwen2.5-1.5B-Instruct agents~\citep{qwen2025technicalreport} on AQUA-RAT~\citep{ling2017AQUARAT} and GSM8K~\citep{cobbe2021training} with $N=10$ agents. Unless otherwise specified, the policy is trained with REINFORCE, a batch-mean baseline, batch size $32$, $50$ policy updates, learning rate $0.1$, dropout $0.3$, and edge-count sparsity coefficient $\lambda_{\mathrm{sp}}=0.1$. The policy architecture is kept fixed: a one-layer, one-head transformer encoder followed by the $\gH\gH^\top$ edge construction described in Section~\ref{t-selforg: methodology}. 

We consider single-agent system, chain-of-thought (CoT)~\citep{wei2022chain}, self-consistency~\citep{wang2023selfconsistency}, SelfOrg$^{\star}$\footnote{SelfOrg$^{\star}$ indicates SelfOrg with a single sequential communication round.}~\citep{tastan2026stochastic}, and topology-learning or pruning methods, including GPTSwarm~\citep{zhuge24gptswarm}, AgentPrune~\citep{zhang2025agentprune}, and G-Designer~\citep{zhang2025gdesigner} as baselines. While the primary metric is accuracy of the final response, we also report the mean edge count and token consumption usage as proxies for the communication burden and inference cost, respectively. 

\vspace{-0.5em}
\subsection{Main Results} 
Table~\ref{tab:main-results} reports the main comparison across AQUA-RAT~\citep{ling2017AQUARAT}, HumanEval~\citep{chen2021humaneval}, and GSM8K~\citep{cobbe2021training}. \method achieves the best average accuracy, $60.90\%$, improving over SelfOrg$^{\star}$ while also obtaining the best average rating. Its gains are strongest on AQUA-RAT and GSM8K while remaining competitive on HumanEval. 

\begin{table}[h]
    \centering
    \caption{Main comparison across AQUA-RAT, HumanEval, and GSM8K. Accuracy is reported as mean $\pm$ std over runs. Rating uses average rating only. Token usage is total prompt plus completion tokens. Lower average rating and lower token usage are better.} 
    \label{tab:main-results}
    \small
    \resizebox{\linewidth}{!}{
        \begin{tabular}{r|ccc|ccr}
            \toprule
            \rowcolor{lightgray}
            \textbf{Method}
                & \textbf{AQUA-RAT}
                & \textbf{HumanEval}
                & \textbf{GSM8K}
                & \textbf{Avg. Acc.}
                & \textbf{Avg. Rating}
                & \textbf{Token Usage} \\
            \midrule
            Single
                & $52.62{\pm}0.99$
                & $50.41{\pm}3.73$
                & $70.07{\pm}1.30$
                & $57.70{\pm}2.01$
                & $5.67$
                & \tokcell{1149321}{100} \\
            CoT
                & $54.46{\pm}1.94$
                & $45.93{\pm}1.27$
                & $70.47{\pm}1.03$
                & $56.95{\pm}1.41$
                & $5.33$
                & \tokcell{1298471}{100} \\
            SC
                & $56.82{\pm}2.27$
                & $12.52{\pm}3.41$
                & $71.53{\pm}0.90$
                & $46.96{\pm}2.20$
                & $5.00$
                & \tokcell{11215153}{78} \\
            SelfOrg$^{\star}$
                & $56.46{\pm}1.28$
                & $\mathbf{52.03{\pm}2.44}$
                & $\underline{72.60{\pm}0.60}$
                & $\underline{60.36{\pm}1.44}$
                & $\underline{2.67}$
                & \tokcell{28405941}{26} \\
            GPTSwarm
                & $55.91{\pm}1.97$
                & $36.79{\pm}1.96$
                & $69.33{\pm}2.61$
                & $54.01{\pm}2.18$
                & $6.33$
                & \tokcell{38021213}{0} \\
            AgentPrune
                & $\underline{57.58{\pm}2.84}$
                & $29.47{\pm}0.93$
                & $71.07{\pm}1.94$
                & $52.71{\pm}1.90$
                & $4.00$
                & \tokcell{31444709}{18} \\
            GDesigner
                & $57.13{\pm}1.74$
                & $28.25{\pm}1.54$
                & $70.00{\pm}1.06$
                & $51.79{\pm}1.44$
                & $5.67$
                & \tokcell{31798540}{17} \\
            \midrule 
            \rowcolor{lightblue} 
            \method
                & $\mathbf{57.74{\pm}2.31}$
                & $\underline{51.42{\pm}0.70}$
                & $\mathbf{73.53{\pm}0.23}$
                & $\mathbf{60.90{\pm}1.08}$
                & $\mathbf{1.33}$
                & \tokcell{18363825}{67} \\
            \bottomrule
        \end{tabular}
    }
\end{table}
% \vspace{-1.5em}
\begin{figure}[h] 
    \centering
    \includegraphics[width=\linewidth]{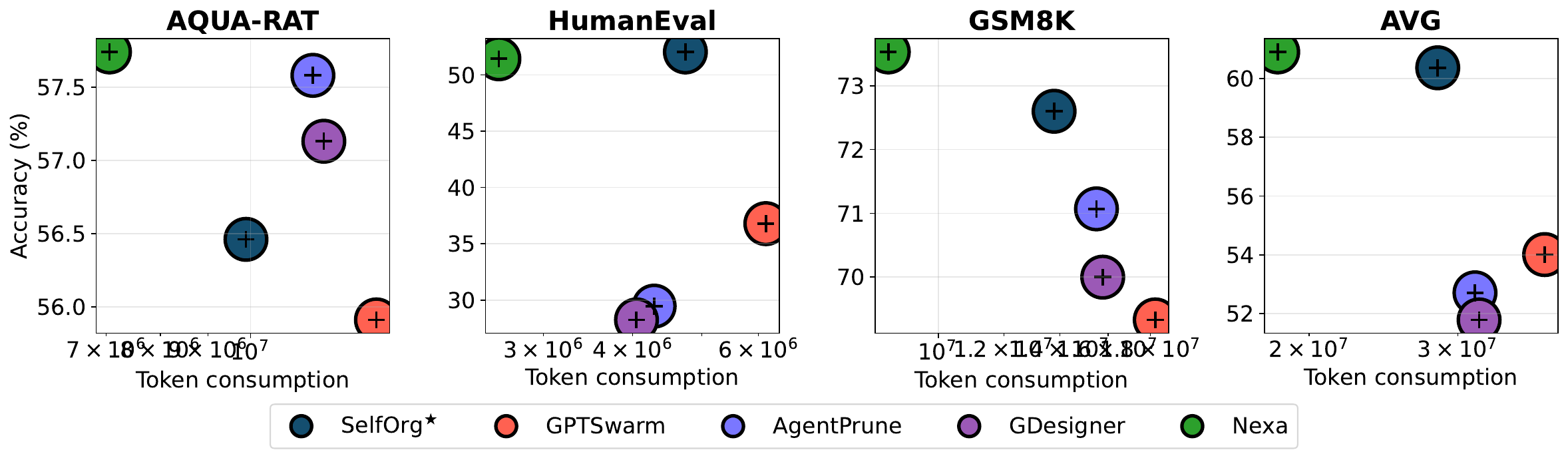}
    \vspace{-1em}
    \caption{Accuracy-cost tradeoff for multi-agent system baselines across three tasks. Each point corresponds to one method, with the x-axis showing total token usage including prompt and completion tokens and the y-axis showing mean accuracy.} % \method achieves the best average accuracy while using substantially fewer tokens than the highest-cost topology-learning baselines, indicating a stronger accuracy-efficiency tradeoff. 
    \label{fig: main-results-token-to-acc} 
    \vspace{-1.5em}
\end{figure}

The efficiency results are central to the comparison. \method uses $18.36$M total tokens, compared with $28.41$M for SelfOrg$^*$, $38.02$M for GPTSwarm, $31.44$M for AgentPrune, and $31.80$M for GDesigner. It therefore reduces token usage by about $35\%$ relative to SelfOrg$^*$ and by more than $50\%$ relative to GPTSwarm, while achieving the highest average accuracy. Figure~\ref{fig: main-results-token-to-acc} makes this tradeoff explicit: in the average panel, \method occupies the favorable region of the accuracy-cost plane, indicating that its improvements are not simply the result of spending more tokens but of selectively invoking communication when the response pool warrants it.

\vspace{-0.5em}
\subsection{Generalizability Across Different Axes} 

\paragraph{Number of agents.}
\begin{wrapfigure}{r}{0.52\linewidth}
    \vspace{-2.25em} 
    \centering
    \includegraphics[width=\linewidth]{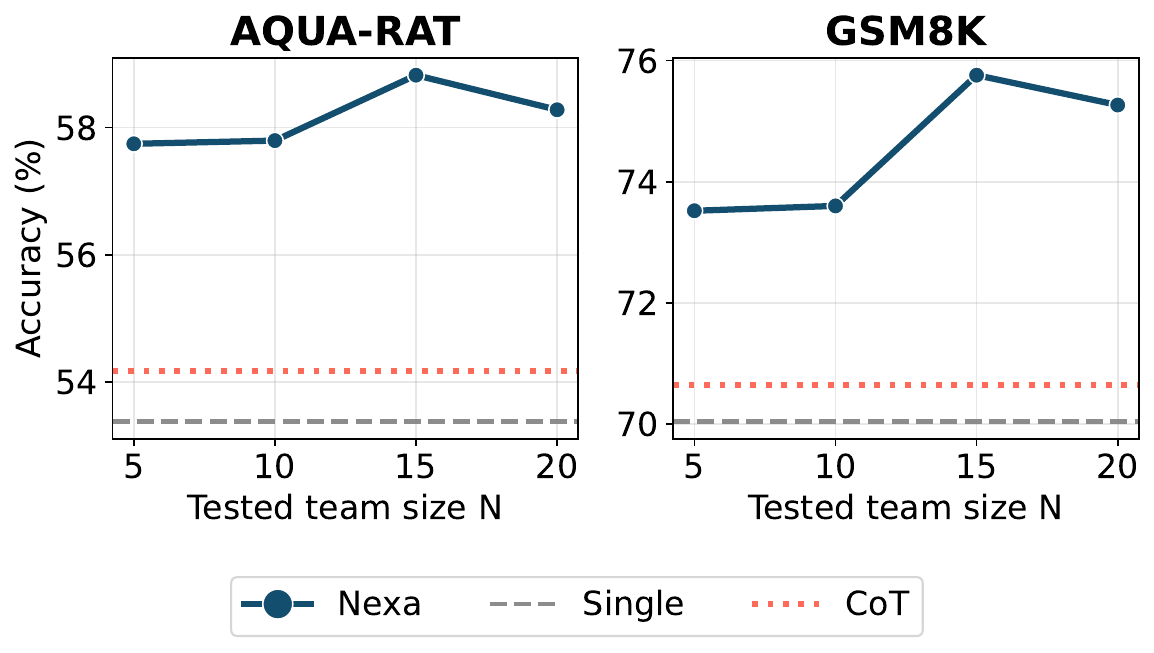}
    \vspace{-1.0em}
    \caption{Agent-count transfer for \method. The policy is trained with $N{=}10$ Qwen2.5-1.5B agents and evaluated without retraining at $N{\in}\{5,\ldots,20\}$.} % Dashed lines denote single-call and chain-of-thought baselines. 
    \label{fig:agent-count-transfer}
    \vspace{-1.em}
\end{wrapfigure}
We first examine generalizability across the number of agents. \method is trained with $N=10$ agents and evaluated without retraining for $N\in\{5,10,15,20\}$, keeping the task and agent backbone fixed. This setting tests whether the learned graph policy behaves as a reusable response-conditioned rule rather than as a memorized topology for a fixed team size. As shown in Figure~\ref{fig:agent-count-transfer}, \method remains above the single-call and chain-of-thought baselines~\citep{wei2022chain} for all tested values of $N$ on both AQUA-RAT and GSM8K. Accuracy peaks at $N=15$ for both tasks, suggesting that the policy can benefit from additional candidate responses beyond the training configuration while still remaining stable when the team size is smaller or larger than~$N=10$.

\vspace{-0.5em}
\paragraph{Task transfer.} We next consider generalizability across tasks while keeping the model family, model size, and training team size fixed. \method is trained with $N=10$ Qwen2.5-1.5B agents on either AQUA-RAT or GSM8K, then evaluated without retraining on both tasks. Figure~\ref{fig:task-transfer} compares same-task training against cross-task training at two tested team sizes, $N=5$ and $N=20$. Across all four settings, the transfer gap remains small: $0.18$ and $0.14$ points on AQUA-RAT, and $0.08$ and $0.05$ points on GSM8K. This suggests that \method may learn a reusable response-conditioned rule rather than merely memorizing a task-specific pattern, although this requires further confirmation under more heterogeneous model families and agent pools. 
\vspace{-1em}
\begin{figure}[h]
    \centering
    \includegraphics[width=\linewidth]{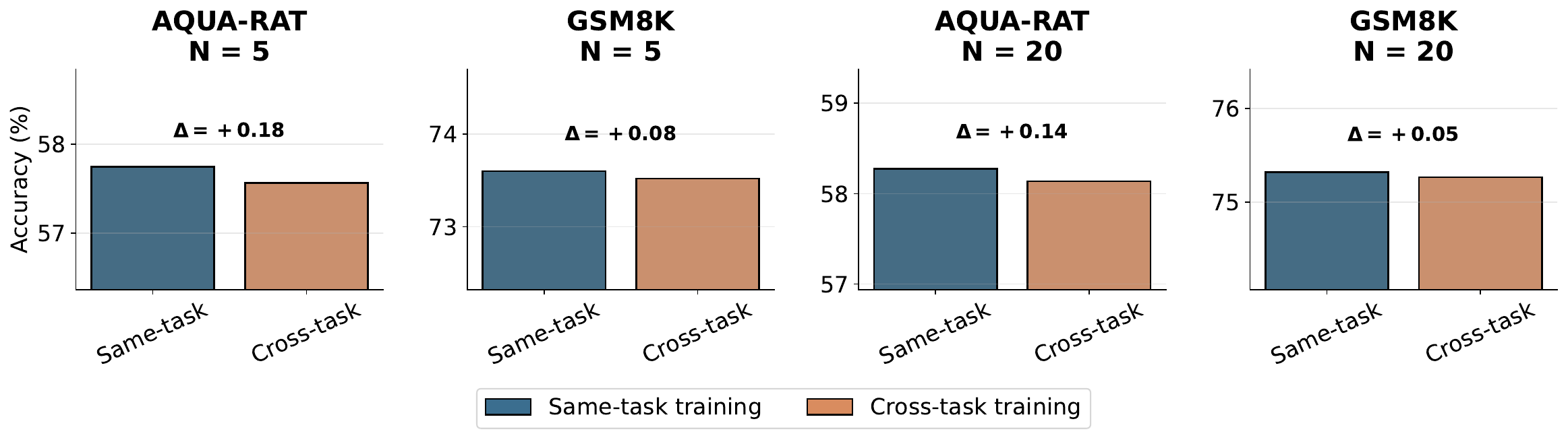}
    \vspace{-1.75em} 
    \caption{Task-transfer comparison for \method on Qwen2.5-1.5B.} % The policy is trained with $N=10$ agents on either the same task as evaluation or the other task, then evaluated without retraining on AQUA-RAT and GSM8K at $N=5$ and $N=20$. Each panel reports accuracy for same-task and cross-task training; $\Delta$ denotes the same-task minus cross-task accuracy gap. % The small gaps across both tasks and team sizes indicate that the learned communication policy transfers across task distributions with little loss. 
    \label{fig:task-transfer}
    % \vspace{-2em}
\end{figure}
\paragraph{Model scale generalizability.}
We then evaluate whether the learned communication policy transfers across model scales. \method is trained using Qwen2.5-1.5B agents and evaluated without retraining on Qwen2.5-7B agents, then compared against a policy trained directly with Qwen2.5-7B agents. As shown in Figure~\ref{fig:model-scale-transfer}, the 1.5B-trained policy closely matches the 7B-trained policy on both tasks: $90.48$ versus $90.52$ on GSM8K, and $76.98$ versus $77.40$ on AQUA-RAT. This suggests that the learned graph policy is not tightly coupled to the competence level of the training backbone and can be reused when deployed with a stronger model.
\begin{wrapfigure}{l}{0.5\linewidth} 
    \vspace{-1em}
    \centering
    \includegraphics[width=\linewidth]{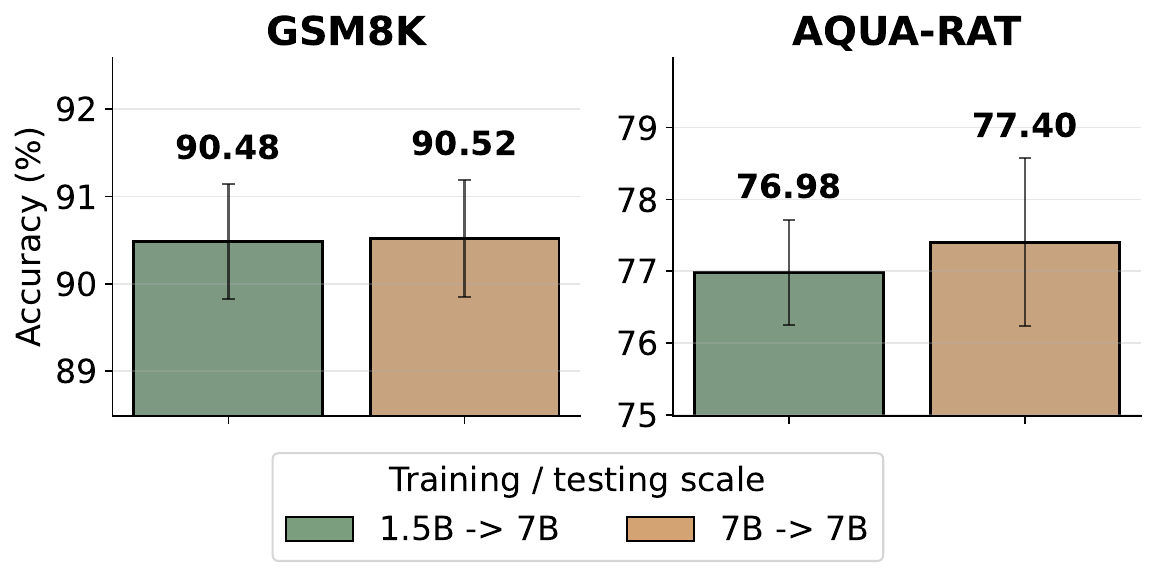}
    \vspace{-1.5em}
    \caption{Model-scale transfer for \method.} % We compare a policy trained on Qwen2.5-1.5B and evaluated on Qwen2.5-7B against a policy trained and evaluated on Qwen2.5-7B. Results are reported on GSM8K and AQUA-RAT with standard deviations. 
    \label{fig:model-scale-transfer}
    \vspace{-0.8em}
\end{wrapfigure}

\vspace{-2.em} 

\paragraph{Model generation transfer.} 
Finally, we evaluate whether the learned communication policy remains usable when the underlying model is updated to a newer generation. \method trained on Qwen2.5-1.5B is evaluated without retraining on Qwen3.5-2B~\citep{qwen3.5} and compared against a policy trained directly on Qwen3.5-2B. At $N=5$, the transferred policy reaches $77.40$, compared with $77.73$ for the target-generation policy. The resulting $0.17$-point gap suggests that an existing policy can remain effective after a model upgrade, reducing the need to retrain the communication controller every time the base model is changed.

\vspace{-1em}
\subsection{How Communication Changes Answers}
\vspace{-0.7em}
\begin{wrapfigure}{r}{0.32\linewidth} 
    \vspace{-3.2em}
    \centering
    \includegraphics[width=\linewidth]{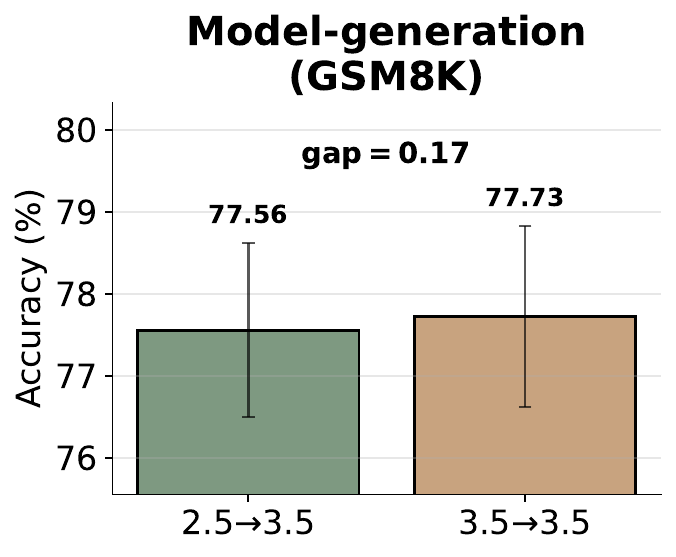}
    \vspace{-1.0em}
    \caption{Model-generation transfer for \method.} % The transferred policy closely matches the target-generation policy, with a small gap. 
    \label{fig:model-generation-transfer}
    \vspace{-1.5em}
\end{wrapfigure}
We further analyze how \method changes answers after communication by decomposing each example according to whether the initial draft (parallel execution responses) and final answer (sequential execution responses) are correct. Figure~\ref{fig:policy-behavior-7b} reports rescue, harm, and preservation rates for Qwen2.5-7B agents on GSM8K. As the tested team size increases from $N=5$ to $N=20$, the rescue rate rises from $19.2\%$ to $23.8\%$, showing that additional agents provide useful opportunities for correcting initially wrong answers. At the same time, harm remains low, between $1.6\%$ and $2.5\%$, while preservation stays above $97.5\%$ across all tested values of $N$. These results suggest that \method does not simply perturb answers through extra communication; it mostly preserves correct predictions while selectively improving initially incorrect ones. 

Additional sparsity diagnostics in Appendix~\ref{app:sparse-communication} show that \method often selects low-edge communication plans, indicating that the learned policy does not rely on dense all-to-all interaction as team size increases. 

\begin{figure}[H] 
    \centering
    \includegraphics[width=0.85\linewidth]{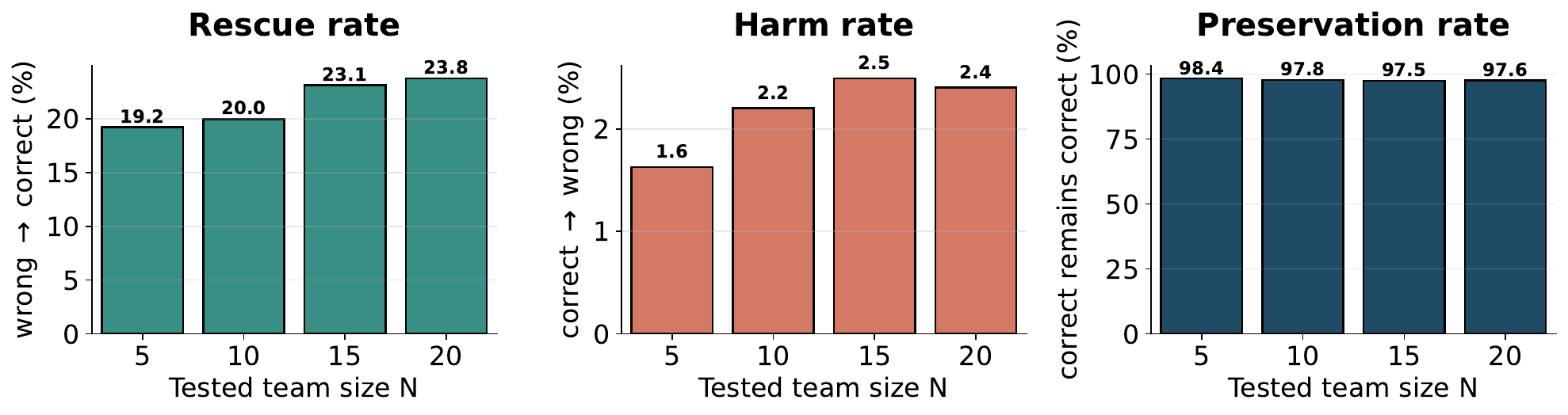}
    \caption{Policy behavior analysis for \method with Qwen2.5-7B agents on GSM8K. Rescue, harm, and preservation rates compare initial draft correctness with final answer correctness after communication.} 
    % Rescue rate measures the fraction of initially incorrect draft answers that become correct after communication, harm rate measures the fraction of initially correct drafts that become incorrect, and preservation rate measures the fraction of initially correct drafts that remain correct. % Across tested team sizes, \method rescues $19.2\%$-$23.8\%$ of initially wrong answers, while harming only $1.6\%$-$2.5\%$ of initially correct answers and preserving more than $97.5\%$ of correct drafts. 
    \label{fig:policy-behavior-7b}
\end{figure}
\vspace{-.5em}

\vspace{-1em}
\subsection{Ablations} 

\paragraph{Policy backbone.} 
\begin{wraptable}{r}{0.55\linewidth}
    \vspace{-1.85em} 
    \centering
    \caption{Backbone ablation on GSM8K with Qwen2.5-1.5B agents. Accuracy is mean $\pm$ std. over three runs.} 
    \label{tab:backbone-ablation}
    \small
    \resizebox{\linewidth}{!}{
        \begin{tabular}{l|cc}
            \toprule
            \rowcolor{lightgray}\textbf{\method (Backbone)} & $\mathbf{N=5}$ & $\mathbf{N=10}$ \\ 
            \midrule
            \method (Transformer) 
                & $72.53{\pm}1.17$ 
                & $75.00{\pm}0.35$ \\
            \method (GNN) 
                & $72.47{\pm}1.68$  
                & $74.87{\pm}1.27$ \\
            \bottomrule
        \end{tabular}
    }
    \vspace{-1.0em}
\end{wraptable}
\method is not tied to a single policy-network backbone. Although our main implementation uses a Transformer to predict response-conditioned communication graphs, the same formulation can be instantiated with other graph-prediction architectures. As one example, we adapt the GNN architecture from GDesigner, originally used for agent-role-specific (and fixed-agent-number) design, to \method's response-conditioned communication graph prediction setting while keeping the rest of the training and communication procedure unchanged. Table~\ref{tab:backbone-ablation} shows that this GNN backbone closely matches the Transformer backbone on GSM8K with Qwen2.5-1.5B agents, suggesting that the core benefit comes from the \method formulation rather than a specific neural backbone. 

\vspace{-.5em}
\paragraph{Policy optimization.} 
\begin{wraptable}{r}{0.48\linewidth}
    \vspace{-1.85em}
    \centering
    \caption{Policy-optimization ablation on AQUA-RAT with $N{=}5$ Qwen2.5-1.5B agents.}
    \label{tab:optimizer-ablation}
    \small
    \resizebox{\linewidth}{!}{
    \begin{tabular}{lcc}
        \toprule
        \rowcolor{lightgray}\textbf{Optimizer} & \textbf{AQUA$\rightarrow$AQUA} & \textbf{GSM8K$\rightarrow$AQUA} \\ 
        \midrule
        GRPO & $57.56 \pm 3.83$ & $57.48 \pm 2.30$ \\
        PG   & $57.74 \pm 2.31$ & $57.56 \pm 1.49$ \\
        \bottomrule
    \end{tabular}
    }
    \vspace{-1.0em}
\end{wraptable}
We also compare the policy-gradient objective used in \method with a GRPO-style alternative. The ablation is conducted on AQUA-RAT with Qwen2.5-1.5B agents at $N=5$, considering both same-task training and cross-task transfer from GSM8K. As shown in Table~\ref{tab:optimizer-ablation}, PG slightly outperforms GRPO in both settings, with $57.74$ versus $57.56$ for AQUA$\rightarrow$AQUA and $57.56$ versus $57.48$ for GSM8K$\rightarrow$AQUA. The gaps are small, indicating that the learned communication policy is not highly sensitive. 

\vspace{-1.0em}
\section{Related Works}
\vspace{-0.75em}
\label{tselforg: literature}

\paragraph{LLM-based multi-agent collaboration.} Multi-agent LLM systems have been studied as role-based societies, conversational workflows, and dynamically routed agent networks. CAMEL instantiates role-playing agents for cooperative problem solving~\citep{li2023camel}, ChatDev organizes specialized agents into staged communicative workflows~\citep{qian2024chatdev}, AutoGen provides a general framework for multi-agent conversations~\citep{wu2024autogen}, and AgentVerse studies collaborative behaviors across agent groups~\citep{chen2024agentverse}. DyLAN adapts the active agent set during task solving~\citep{liu2024dylan}, while multi-agent debate methods use disagreement to improve reasoning or factuality~\citep{du2023improving,liang2024MAD}. Multiagent finetuning further studies whether diverse reasoning chains can improve a base model through self-improvement~\citep{subramaniam2025multiagent}. 
These systems show that collaboration can improve reasoning, but they typically require a chosen communication protocol, a task-specific workflow, or an explicit judging mechanism. \method instead begins with independent responses and learns whether any sequential communication should occur at all. 
\vspace{-.5em}
\paragraph{Communication topology and workflow design.} Several recent methods treat agent orchestration as a graph or workflow optimization problem. GPTSwarm represents language agents as optimizable computational graphs~\citep{zhuge24gptswarm}; AgentPrune removes unnecessary communication to reduce costs~\citep{zhang2025agentprune}; G-Designer learns communication topologies with graph neural networks~\citep{zhang2025gdesigner}; and MacNet studies scaling laws for LLM-based multi-agent collaboration~\citep{qian2025scaling}. Related work also explores training LLMs to construct multi-agent systems~\citep{ye2025masgpt}, automated agentic workflow generation~\citep{hu2025automated,zhang2025aflow}, decentralized evolutionary coordination~\citep{yang2025agentnet}, self-evolving agent profiles~\citep{lu2024morphagent}, heterogeneous multi-agent systems~\citep{ye2025xmas}, and unified experimental platforms for multi-agent evaluation~\citep{ye2025maslab}. \method is closest in spirit to graph-based topology learning but differs in three ways: the graph is conditioned on the realized response pool rather than only on a task or role template; the empty graph is a valid decision corresponding to pure parallel execution; and sparsity is controlled directly through an edge-count penalty in the task reward. 
\vspace{-.5em}
\paragraph{Judge-free aggregation and acyclicity.} 
Response selection and ensemble fusion are often performed by majority voting, learned rankers, or generative fusion models such as LLM-Blender~\citep{jiang2023llmblender}; other systems introduce credibility scores or adversary-resistant judges~\citep{ebrahimi2025adversary}. SelfOrg takes a different route by estimating response contribution from semantic embeddings and using that signal to organize communication without an external judge~\citep{tastan2026stochastic}. Its contribution score is motivated by Shapley-style valuation~\citep{shapley1953value, tastan2025aequa} and can be computed from sentence embeddings~\citep{reimers2019sentencebert}. \method keeps this judge-free contribution ordering but replaces the stochastic self-organization rule with a trainable structure policy. The method, therefore, preserves the stable ordering principle of SelfOrg while learning the forward edges that determine whether and where refinement should happen. 
\vspace{-.5em}
\section{Conclusion}
\vspace{-.5em}
We introduced \method, a response-conditioned policy that bridges parallel and sequential multi-agent execution by learning sparse acyclic communication graphs from initial agent drafts. The method remains lightweight and judge-free, can reduce to pure parallel execution, and improves the accuracy-cost tradeoff while transferring across tasks, agent counts, and model settings.

\clearpage 

\bibliographystyle{plainnat}
\bibliography{references}

@inproceedings{tastan2026stochastic,
  title = {Stochastic Self-Organization in Multi-Agent Systems},
  author = {Tastan, Nurbek and Horv{\'a}th, Samuel and Nandakumar, Karthik},
  booktitle = {The Fourteenth International Conference on Learning Representations},
  year = {2026},
  url = {https://openreview.net/forum?id=rS3Jb9AAej},
}

@InProceedings{zhuge24gptswarm,
    title = 	 {{GPTS}warm: Language Agents as Optimizable Graphs},
    author =       {Zhuge, Mingchen and Wang, Wenyi and Kirsch, Louis and Faccio, Francesco and Khizbullin, Dmitrii and Schmidhuber, J\"{u}rgen},
    booktitle = 	 {Proceedings of the 41st International Conference on Machine Learning},
    pages = 	 {62743--62767},
    year = 	 {2024},
    editor = 	 {Salakhutdinov, Ruslan and Kolter, Zico and Heller, Katherine and Weller, Adrian and Oliver, Nuria and Scarlett, Jonathan and Berkenkamp, Felix},
    volume = 	 {235},
    series = 	 {Proceedings of Machine Learning Research},
    month = 	 {21--27 Jul},
    publisher =    {PMLR},
    url = 	 {https://proceedings.mlr.press/v235/zhuge24a.html}
}

@inproceedings{zhang2025agentprune,
    title={Cut the Crap: An Economical Communication Pipeline for {LLM}-based Multi-Agent Systems},
    author={Guibin Zhang and Yanwei Yue and Zhixun Li and Sukwon Yun and Guancheng Wan and Kun Wang and Dawei Cheng and Jeffrey Xu Yu and Tianlong Chen},
    booktitle={The Thirteenth International Conference on Learning Representations},
    year={2025},
    url={https://openreview.net/forum?id=LkzuPorQ5L}
}

@inproceedings{zhang2025gdesigner,
    title={G-Designer: Architecting Multi-agent Communication Topologies via Graph Neural Networks},
    author={Guibin Zhang and Yanwei Yue and Xiangguo Sun and Guancheng Wan and Miao Yu and Junfeng Fang and Kun Wang and Tianlong Chen and Dawei Cheng},
    booktitle={Forty-second International Conference on Machine Learning},
    year={2025},
    url={https://openreview.net/forum?id=LpE54NUnmO}
}

@inproceedings{wu2024autogen,
    title={AutoGen: Enabling Next-Gen {LLM} Applications via Multi-Agent Conversations},
    author={Qingyun Wu and Gagan Bansal and Jieyu Zhang and Yiran Wu and Beibin Li and Erkang Zhu and Li Jiang and Xiaoyun Zhang and Shaokun Zhang and Jiale Liu and Ahmed Hassan Awadallah and Ryen W White and Doug Burger and Chi Wang},
    booktitle={First Conference on Language Modeling},
    year={2024},
    url={https://openreview.net/forum?id=BAakY1hNKS}
}

@article{ebrahimi2025adversary,
  title={An Adversary-Resistant Multi-Agent LLM System via Credibility Scoring},
  author={Ebrahimi, Sana and Dehghankar, Mohsen and Asudeh, Abolfazl},
  journal={arXiv preprint arXiv:2505.24239},
  year={2025}
}

@article{ye2025xmas,
  title={X-MAS: Towards Building Multi-Agent Systems with Heterogeneous LLMs},
  author={Ye, Rui and Liu, Xiangrui and Wu, Qimin and Pang, Xianghe and Yin, Zhenfei and Bai, Lei and Chen, Siheng},
  journal={arXiv preprint arXiv:2505.16997},
  year={2025}
}

@inproceedings{ye2025masgpt,
    title={{MAS}-{GPT}: Training {LLM}s to Build {LLM}-based Multi-Agent Systems},
    author={Rui Ye and Shuo Tang and Rui Ge and Yaxin Du and Zhenfei Yin and Siheng Chen and Jing Shao},
    booktitle={Forty-second International Conference on Machine Learning},
    year={2025},
    url={https://openreview.net/forum?id=3CiSpY3QdZ}
}

@article{ye2025maslab,
  title={MASLab: A Unified and Comprehensive Codebase for LLM-based Multi-Agent Systems},
  author={Ye, Rui and Huang, Keduan and Wu, Qimin and Cai, Yuzhu and Jin, Tian and Pang, Xianghe and Liu, Xiangrui and Su, Jiaqi and Qian, Chen and Tang, Bohan and others},
  journal={arXiv preprint arXiv:2505.16988},
  year={2025}
}

@inproceedings{qian2025scaling,
    title={Scaling Large Language Model-based Multi-Agent Collaboration},
    author={Chen Qian and Zihao Xie and YiFei Wang and Wei Liu and Kunlun Zhu and Hanchen Xia and Yufan Dang and Zhuoyun Du and Weize Chen and Cheng Yang and Zhiyuan Liu and Maosong Sun},
    booktitle={The Thirteenth International Conference on Learning Representations},
    year={2025},
    url={https://openreview.net/forum?id=K3n5jPkrU6}
}

@misc{qwen2025technicalreport,
      title={Qwen2.5 Technical Report}, 
      author={Qwen and : and An Yang and Baosong Yang and Beichen Zhang and Binyuan Hui and Bo Zheng and Bowen Yu and Chengyuan Li and Dayiheng Liu and Fei Huang and Haoran Wei and Huan Lin and Jian Yang and Jianhong Tu and Jianwei Zhang and Jianxin Yang and Jiaxi Yang and Jingren Zhou and Junyang Lin and Kai Dang and Keming Lu and Keqin Bao and Kexin Yang and Le Yu and Mei Li and Mingfeng Xue and Pei Zhang and Qin Zhu and Rui Men and Runji Lin and Tianhao Li and Tianyi Tang and Tingyu Xia and Xingzhang Ren and Xuancheng Ren and Yang Fan and Yang Su and Yichang Zhang and Yu Wan and Yuqiong Liu and Zeyu Cui and Zhenru Zhang and Zihan Qiu},
      year={2025},
      eprint={2412.15115},
      archivePrefix={arXiv},
      primaryClass={cs.CL},
      url={https://arxiv.org/abs/2412.15115}, 
}

@inproceedings{li2023camel,
    title={{CAMEL}: Communicative Agents for ''Mind'' Exploration of Large Language Model Society},
    author={Guohao Li and Hasan Abed Al Kader Hammoud and Hani Itani and Dmitrii Khizbullin and Bernard Ghanem},
    booktitle={Thirty-seventh Conference on Neural Information Processing Systems},
    year={2023},
    url={https://openreview.net/forum?id=3IyL2XWDkG}
}

@inproceedings{qian2024chatdev,
    title = "{C}hat{D}ev: Communicative Agents for Software Development",
    author = "Qian, Chen  and
      Liu, Wei  and
      Liu, Hongzhang  and
      Chen, Nuo  and
      Dang, Yufan  and
      Li, Jiahao  and
      Yang, Cheng  and
      Chen, Weize  and
      Su, Yusheng  and
      Cong, Xin  and
      Xu, Juyuan  and
      Li, Dahai  and
      Liu, Zhiyuan  and
      Sun, Maosong",
    editor = "Ku, Lun-Wei  and
      Martins, Andre  and
      Srikumar, Vivek",
    booktitle = "Proceedings of the 62nd Annual Meeting of the Association for Computational Linguistics (Volume 1: Long Papers)",
    month = aug,
    year = "2024",
    address = "Bangkok, Thailand",
    publisher = "Association for Computational Linguistics",
    url = "https://aclanthology.org/2024.acl-long.810/",
    doi = "10.18653/v1/2024.acl-long.810",
    pages = "15174--15186"
}

@article{wei2022chain,
  title={Chain-of-thought prompting elicits reasoning in large language models},
  author={Wei, Jason and Wang, Xuezhi and Schuurmans, Dale and Bosma, Maarten and Xia, Fei and Chi, Ed and Le, Quoc V and Zhou, Denny and others},
  journal={Advances in neural information processing systems},
  volume={35},
  pages={24824--24837},
  year={2022}
}

@inproceedings{reimers2019sentencebert,
  title = "Sentence-BERT: Sentence Embeddings using Siamese BERT-Networks",
  author = "Reimers, Nils and Gurevych, Iryna",
  booktitle = "Proceedings of the 2019 Conference on Empirical Methods in Natural Language Processing",
  month = "11",
  year = "2019",
  publisher = "Association for Computational Linguistics",
  url = "https://arxiv.org/abs/1908.10084",
}

@InProceedings{tastan2025aequa,
  title = 	 {{Aequa: Fair Model Rewards in Collaborative Learning via Slimmable Networks}},
  author =       {Tastan, Nurbek and Horv\'{a}th, Samuel and Nandakumar, Karthik},
  booktitle = 	 {Proceedings of the 42nd International Conference on Machine Learning},
  pages = 	 {59210--59236},
  year = 	 {2025},
  editor = 	 {Singh, Aarti and Fazel, Maryam and Hsu, Daniel and Lacoste-Julien, Simon and Berkenkamp, Felix and Maharaj, Tegan and Wagstaff, Kiri and Zhu, Jerry},
  volume = 	 {267},
  series = 	 {Proceedings of Machine Learning Research},
  month = 	 {13--19 Jul},
  publisher =    {PMLR},
  url = 	 {https://proceedings.mlr.press/v267/tastan25a.html}
}

@incollection{shapley1953value, 
  title = {A Value for n-Person Games},
  author = {Shapley, Lloyd S},
  booktitle = {Contributions to the Theory of Games II},
  editor = {Kuhn, Harold W. and Tucker, Albert W.},
  pages = {307--317},
  year = {1953},
  publisher = {Princeton University Press},
  address = {Princeton}
}

@inproceedings{du2023improving,
  title={Improving factuality and reasoning in language models through multiagent debate},
  author={Du, Yilun and Li, Shuang and Torralba, Antonio and Tenenbaum, Joshua B and Mordatch, Igor},
  booktitle={Forty-first International Conference on Machine Learning},
  year={2023}
}

@inproceedings{subramaniam2025multiagent,
    title={Multiagent Finetuning: Self Improvement with Diverse Reasoning Chains},
    author={Vighnesh Subramaniam and Yilun Du and Joshua B. Tenenbaum and Antonio Torralba and Shuang Li and Igor Mordatch},
    booktitle={The Thirteenth International Conference on Learning Representations},
    year={2025},
    url={https://openreview.net/forum?id=JtGPIZpOrz}
}

@inproceedings{liang2024MAD,
    title = "Encouraging Divergent Thinking in Large Language Models through Multi-Agent Debate",
    author = "Liang, Tian  and
      He, Zhiwei  and
      Jiao, Wenxiang  and
      Wang, Xing  and
      Wang, Yan  and
      Wang, Rui  and
      Yang, Yujiu  and
      Shi, Shuming  and
      Tu, Zhaopeng",
    editor = "Al-Onaizan, Yaser  and
      Bansal, Mohit  and
      Chen, Yun-Nung",
    booktitle = "Proceedings of the 2024 Conference on Empirical Methods in Natural Language Processing",
    month = nov,
    year = "2024",
    address = "Miami, Florida, USA",
    publisher = "Association for Computational Linguistics",
    url = "https://aclanthology.org/2024.emnlp-main.992/",
    doi = "10.18653/v1/2024.emnlp-main.992",
    pages = "17889--17904"
}

@inproceedings{chen2024agentverse,
    title={AgentVerse: Facilitating Multi-Agent Collaboration and Exploring Emergent Behaviors},
    author={Weize Chen and Yusheng Su and Jingwei Zuo and Cheng Yang and Chenfei Yuan and Chi-Min Chan and Heyang Yu and Yaxi Lu and Yi-Hsin Hung and Chen Qian and Yujia Qin and Xin Cong and Ruobing Xie and Zhiyuan Liu and Maosong Sun and Jie Zhou},
    booktitle={The Twelfth International Conference on Learning Representations},
    year={2024},
    url={https://openreview.net/forum?id=EHg5GDnyq1}
}

@inproceedings{liu2024dylan,
    title={A Dynamic {LLM}-Powered Agent Network for Task-Oriented Agent Collaboration},
    author={Zijun Liu and Yanzhe Zhang and Peng Li and Yang Liu and Diyi Yang},
    booktitle={First Conference on Language Modeling},
    year={2024},
    url={https://openreview.net/forum?id=XII0Wp1XA9}
}

@inproceedings{hu2025automated,
    title={Automated Design of Agentic Systems},
    author={Shengran Hu and Cong Lu and Jeff Clune},
    booktitle={The Thirteenth International Conference on Learning Representations},
    year={2025},
    url={https://openreview.net/forum?id=t9U3LW7JVX}
}

@inproceedings{zhang2025aflow,
    title={{AF}low: Automating Agentic Workflow Generation},
    author={Jiayi Zhang and Jinyu Xiang and Zhaoyang Yu and Fengwei Teng and Xiong-Hui Chen and Jiaqi Chen and Mingchen Zhuge and Xin Cheng and Sirui Hong and Jinlin Wang and Bingnan Zheng and Bang Liu and Yuyu Luo and Chenglin Wu},
    booktitle={The Thirteenth International Conference on Learning Representations},
    year={2025},
    url={https://openreview.net/forum?id=z5uVAKwmjf}
}

@inproceedings{jiang2023llmblender,
    title = "{LLM}-Blender: Ensembling Large Language Models with Pairwise Ranking and Generative Fusion",
    author = "Jiang, Dongfu  and
      Ren, Xiang  and
      Lin, Bill Yuchen",
    editor = "Rogers, Anna  and
      Boyd-Graber, Jordan  and
      Okazaki, Naoaki",
    booktitle = "Proceedings of the 61st Annual Meeting of the Association for Computational Linguistics (Volume 1: Long Papers)",
    month = jul,
    year = "2023",
    address = "Toronto, Canada",
    publisher = "Association for Computational Linguistics",
    url = "https://aclanthology.org/2023.acl-long.792/",
    doi = "10.18653/v1/2023.acl-long.792",
    pages = "14165--14178"
}

@inproceedings{hendrycks2021measuring,
title={Measuring Mathematical Problem Solving With the {MATH} Dataset},
author={Dan Hendrycks and Collin Burns and Saurav Kadavath and Akul Arora and Steven Basart and Eric Tang and Dawn Song and Jacob Steinhardt},
booktitle={Thirty-fifth Conference on Neural Information Processing Systems Datasets and Benchmarks Track (Round 2)},
year={2021},
url={https://openreview.net/forum?id=7Bywt2mQsCe}
}

@InProceedings{gao23PAL,
  title = 	 {{PAL}: Program-aided Language Models},
  author =       {Gao, Luyu and Madaan, Aman and Zhou, Shuyan and Alon, Uri and Liu, Pengfei and Yang, Yiming and Callan, Jamie and Neubig, Graham},
  booktitle = 	 {Proceedings of the 40th International Conference on Machine Learning},
  pages = 	 {10764--10799},
  year = 	 {2023},
  editor = 	 {Krause, Andreas and Brunskill, Emma and Cho, Kyunghyun and Engelhardt, Barbara and Sabato, Sivan and Scarlett, Jonathan},
  volume = 	 {202},
  series = 	 {Proceedings of Machine Learning Research},
  month = 	 {23--29 Jul},
  publisher =    {PMLR},
  url = 	 {https://proceedings.mlr.press/v202/gao23f.html}
}

@inproceedings{ling2017AQUARAT,
    title = "Program Induction by Rationale Generation: Learning to Solve and Explain Algebraic Word Problems",
    author = "Ling, Wang  and
      Yogatama, Dani  and
      Dyer, Chris  and
      Blunsom, Phil",
    editor = "Barzilay, Regina  and
      Kan, Min-Yen",
    booktitle = "Proceedings of the 55th Annual Meeting of the Association for Computational Linguistics (Volume 1: Long Papers)",
    month = jul,
    year = "2017",
    address = "Vancouver, Canada",
    publisher = "Association for Computational Linguistics",
    url = "https://aclanthology.org/P17-1015/",
    doi = "10.18653/v1/P17-1015",
    pages = "158--167"
}

@article{cobbe2021training,
  title={Training verifiers to solve math word problems},
  author={Cobbe, Karl and Kosaraju, Vineet and Bavarian, Mohammad and Chen, Mark and Jun, Heewoo and Kaiser, Lukasz and Plappert, Matthias and Tworek, Jerry and Hilton, Jacob and Nakano, Reiichiro and others},
  journal={arXiv preprint arXiv:2110.14168},
  year={2021}
}

@misc{chen2025xverifyefficientanswerverifier,
      title={xVerify: Efficient Answer Verifier for Reasoning Model Evaluations}, 
      author={Ding Chen and Qingchen Yu and Pengyuan Wang and Wentao Zhang and Bo Tang and Feiyu Xiong and Xinchi Li and Minchuan Yang and Zhiyu Li},
      year={2025},
      eprint={2504.10481},
      archivePrefix={arXiv},
      primaryClass={cs.CL},
      url={https://arxiv.org/abs/2504.10481}, 
}

@article{chen2021humaneval,
  author       = {Mark Chen and Jerry Tworek and Heewoo Jun and Qiming Yuan and Henrique Pond{\'{e}} de Oliveira Pinto and Jared Kaplan and Harri Edwards and Yuri Burda and Nicholas Joseph and Greg Brockman and Alex Ray and Raul Puri and Gretchen Krueger and Michael Petrov and Heidy Khlaaf and Girish Sastry and Pamela Mishkin and Brooke Chan and Scott Gray and Nick Ryder and Mikhail Pavlov and Alethea Power and Lukasz Kaiser and Mohammad Bavarian and Clemens Winter and Philippe Tillet and Felipe Petroski Such and Dave Cummings and Matthias Plappert and Fotios Chantzis and Elizabeth Barnes and Ariel Herbert{-}Voss and William Hebgen Guss and Alex Nichol and Alex Paino and Nikolas Tezak and Jie Tang and Igor Babuschkin and Suchir Balaji and Shantanu Jain and William Saunders and Christopher Hesse and Andrew N. Carr and Jan Leike and Joshua Achiam and Vedant Misra and Evan Morikawa and Alec Radford and Matthew Knight and Miles Brundage and Mira Murati and Katie Mayer and Peter Welinder and Bob McGrew and Dario Amodei and Sam McCandlish and Ilya Sutskever and Wojciech Zaremba}, 
  title        = {Evaluating Large Language Models Trained on Code},
  journal      = {CoRR},
  volume       = {abs/2107.03374},
  year         = {2021},
  url          = {https://arxiv.org/abs/2107.03374},
  eprinttype    = {arXiv},
  eprint       = {2107.03374},
  bibsource    = {dblp computer science bibliography, https://dblp.org}
}

@inproceedings{wang2023selfconsistency,
    title={Self-Consistency Improves Chain of Thought Reasoning in Language Models},
    author={Xuezhi Wang and Jason Wei and Dale Schuurmans and Quoc V Le and Ed H. Chi and Sharan Narang and Aakanksha Chowdhery and Denny Zhou},
    booktitle={The Eleventh International Conference on Learning Representations },
    year={2023},
    url={https://openreview.net/forum?id=1PL1NIMMrw}
}

@inproceedings{yang2025agentnet,
    title={AgentNet: Decentralized Evolutionary Coordination for {LLM}-based Multi-Agent Systems},
    author={Yingxuan Yang and Huacan Chai and Shuai Shao and Yuanyi Song and Siyuan Qi and Renting Rui and Weinan Zhang},
    booktitle={The Thirty-ninth Annual Conference on Neural Information Processing Systems},
    year={2025},
    url={https://openreview.net/forum?id=tXqLxHlb8Z}
}

@article{lu2024morphagent,
  title={Morphagent: Empowering agents through self-evolving profiles and decentralized collaboration},
  author={Lu, Siyuan and Shao, Jiaqi and Luo, Bing and Lin, Tao},
  journal={arXiv preprint arXiv:2410.15048},
  year={2024}
}

@inproceedings{vaswani2017attention,
 author = {Vaswani, Ashish and Shazeer, Noam and Parmar, Niki and Uszkoreit, Jakob and Jones, Llion and Gomez, Aidan N and Kaiser, \L ukasz and Polosukhin, Illia},
 booktitle = {Advances in Neural Information Processing Systems},
 editor = {I. Guyon and U. Von Luxburg and S. Bengio and H. Wallach and R. Fergus and S. Vishwanathan and R. Garnett},
 pages = {},
 publisher = {Curran Associates, Inc.},
 title = {Attention is All you Need},
 url = {https://proceedings.neurips.cc/paper_files/paper/2017/file/3f5ee243547dee91fbd053c1c4a845aa-Paper.pdf},
 volume = {30},
 year = {2017}
}

@misc{qwen3.5,
    title  = {{Qwen3.5}: Towards Native Multimodal Agents},
    author = {{Qwen Team}},
    month  = {February},
    year   = {2026},
    url    = {https://qwen.ai/blog?id=qwen3.5}
}

\clearpage 
\appendix

\section{Limitations}
\label{app:limitations}

\method is evaluated primarily on reasoning and programming benchmarks where answer correctness can be measured reliably. This focus allows controlled comparisons across task, agent count, model scale, and model generation, but leaves broader open-ended settings such as long-form generation, interactive tool use, and multi-turn planning as natural directions for future evaluation. 

The method also depends on response embeddings. If the embedding model fails to capture task-relevant differences between candidate answers, the contribution ordering and graph policy may miss useful communication paths. For the goodness of the selected embedding model, we refer the reader to~\citep{tastan2026stochastic}. 

Finally, \method deliberately uses one parallel draft round as the evidence-gathering stage for deciding whether communication is needed. This makes the sequential part selective and often sparse, but it also means that the initial agent pool size remains an important efficiency knob. Future extensions could combine \method with adaptive agent selection so that both the number of initial drafts and the communication graph are chosen instance by instance.

\section{Algorithm}

\begin{algorithm}[h]
    \caption{\method}
    \label{alg:tselforg}
    \begin{algorithmic}[1]
        \Require Query $\gQ$, agents $\{\gA_n\}_{n=1}^N$, encoder $f$, policy $p_\theta(\gE \mid \gX,\pi)$ 
        \For{$n=1$ to $N$}
            \State $\gR_n^{(0)}\gets\gA_n(\gQ)$, \quad $r_n\gets f(\gR_n^{(0)})$ 
        \EndFor
        \State $r_{\rm avg}\gets \frac{1}{N}\sum_n r_n$, \quad $\psi_n\gets\cos(r_n,r_{\rm avg})$
        \State $\pi\gets\argsort(\{\psi_n\}_{n=1}^N;\mathrm{desc})$ %, \quad $\gX\gets[r_1,\dots,r_N]^\top$
        \State $\gX \gets [r_1,\dots,r_N]^{\top}$, \quad $\gH \gets \operatorname{Enc}_{\theta}(\gX)$
        % \State $\gH \gets \operatorname{Enc}_{\theta}(\gX)$ 
        \State Compute masked forward logits $\ell_{m \to n}$ % for all feasible forward pairs 
        \State Sample/decode $\gE\sim p_\theta(\cdot\mid \gX,\pi)$ 
        \State \textbf{if} $\gE=\varnothing$ \textbf{then} return centroid response from $\{\gR_n^{(0)}\}_{n=1}^N$
        % \If{$\gE = \varnothing$}
        %     \State \Return centroid-selected response from $\{\gR_n^{(0)}\}_{n=1}^{N}$
        % \EndIf
        \For{each node $n$ in order $\pi$}
            \State $\Pa(n) \gets \{m:(m \to n)\in \gE\}$
            % \State $\gR_n^{(1)} \gets \gA_n\!\left(\gQ,\{\gR_m^{(\star)}:m\in\Pa(n)\}\right)$ \textbf{if} $\Pa(n)\neq\varnothing$, \textbf{else} $\gR_n^{(0)}$
            
            \State $\gR_n^{(1)}\gets 
            \begin{cases}
                \gA_n\!\left(\gQ,\{\gR_m^{(\star)}:m\in\Pa(n)\}\right), & \Pa(n)\neq\varnothing,\\
                \gR_n^{(0)}, & \text{otherwise.}
            \end{cases}$
            
            % \If{$\Pa(n)\neq\varnothing$}
            %     \State $\gR_n^{(1)} \gets \gA_n\!\left(\gQ,\{\gR_m^{(\star)}:m\in\Pa(n)\}\right)$
            % \Else
            %     \State $\gR_n^{(1)} \gets \gR_n^{(0)}$
            % \EndIf
        \EndFor
        \State \Return centroid response from $\{\gR_n^{(1)}\}_{n=1}^N$ 
        % \State \Return centroid-selected response from $\{\gR_n^{(1)}\}_{n=1}^{N}$
    \end{algorithmic}
\end{algorithm}

\section{Policy Behavior with Smaller Agents} 

We observe the same qualitative behavior with Qwen2.5-1.5B agents. As shown in Figure~\ref{fig:policy-behavior-15b}, the rescue rate increases from $15.6\%$ at $N=5$ to $20.6\%$ at $N=20$, while the preservation rate remains above $92.8\%$ for all tested team sizes. Harm is higher than in the 7B setting (refer to Figure~\ref{fig:policy-behavior-7b}), ranging from $5.6\%$ to $7.2\%$, which is expected given the weaker base agents. Nevertheless, the policy still improves a meaningful fraction of initially wrong answers while preserving the vast majority of initially correct ones, indicating that the same communication behavior appears even with smaller models.

\begin{figure}[h] 
    \centering
    \includegraphics[width=0.85\linewidth]{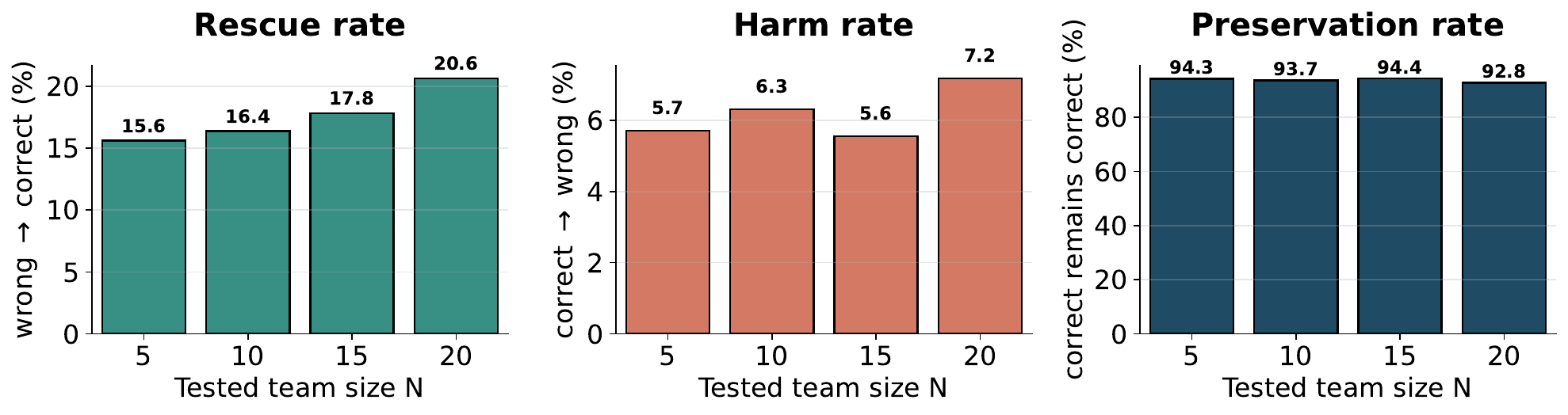}
    \caption{Policy behavior analysis for \method with Qwen2.5-1.5B agents on GSM8K. Rescue, harm, and preservation rates are computed by comparing each initial draft answer with the final answer after communication. \method rescues $15.6\%$-$20.6\%$ of initially wrong answers while preserving $92.8\%$-$94.4\%$ of initially correct answers across tested team sizes.} 
    \label{fig:policy-behavior-15b}
\end{figure}

\section{Communication Sparsity}
\label{app:sparse-communication}

Figure~\ref{fig:sparse-communication-appendix} reports the fraction of low-edge communication plans produced by \method on GSM8K. For Qwen2.5-1.5B, low-edge plans occur in $35.0\%$ of examples at $N=5$ and remain near $47$--$50\%$ for larger tested team sizes. For Qwen2.5-7B, the fraction increases from $43.4\%$ at $N=5$ to more than $70\%$ at $N=15$ and $N=20$. These results suggest that increasing the number of agents does not force dense communication; the learned policy often selects sparse interaction patterns. 

We also observe that, as we scale the capability of the backbone, it leads to more frequent sparse communication than less capable or weaker backbone, indicating that as individual agents become more capable, the policy can rely on fewer communication edges.

\begin{figure}[H] 
    \centering
    \includegraphics[width=0.58\linewidth]{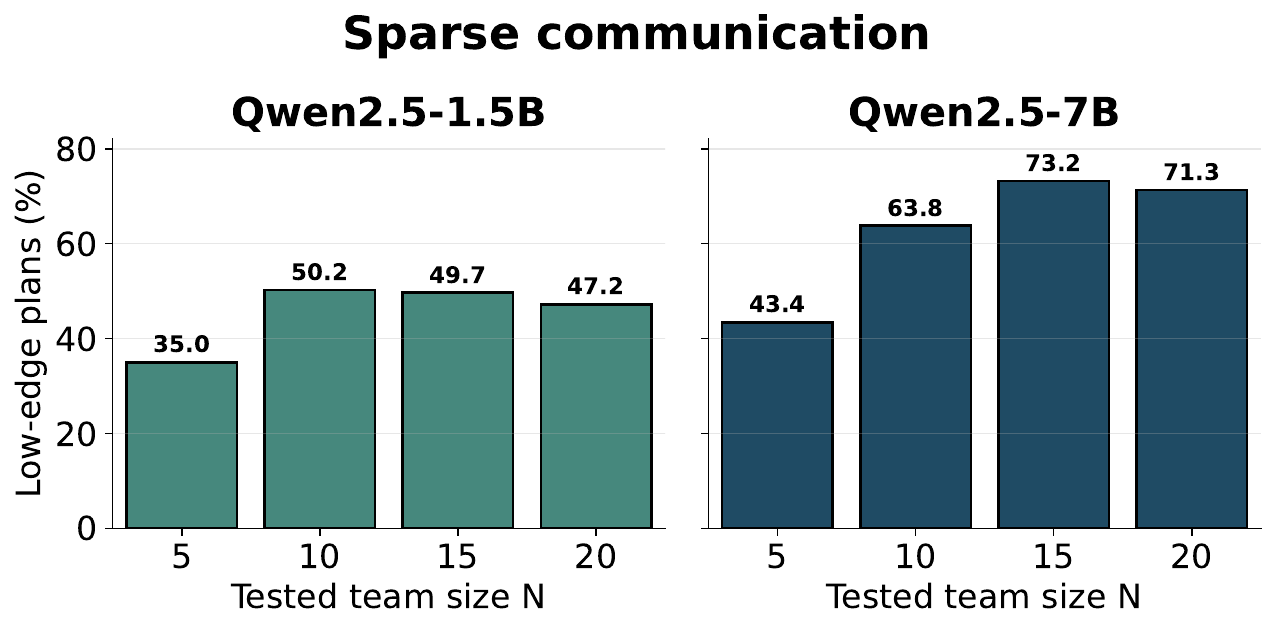}
    \caption{Communication sparsity for \method on GSM8K with Qwen2.5-1.5B and Qwen2.5-7B agents. We report the fraction of examples whose predicted communication graph uses at most half of the possible edges. Across both model sizes, \method frequently selects low-edge plans, indicating that the learned policy does not rely on dense all-to-all communication as team size increases.} 
    \label{fig:sparse-communication-appendix}
\end{figure}

\section{Experimental Settings}

\paragraph{Compute resources.} All experiments were run on NVIDIA A100 40GB GPUs. The same GPU class was used both to serve the LLM agents during multi-agent inference and to train the \method policy.

\paragraph{Backbone ablation setting.} For the policy-backbone ablation, we restrict training to 5 epochs for both the Transformer and GNN variants to keep the comparison controlled and computationally lightweight. Both variants use the same \method training objective, the same GSM8K training setting, Qwen2.5-1.5B agents, sampled graph plans at evaluation, temperature $0.5$, and three random repeats. The only changed component is the policy-network backbone used to score response-conditioned communication graphs. The GNN variant adapts the graph neural architecture from GDesigner from agent-role-specific graph design to \method's response-conditioned graph prediction setting. 

\paragraph{Policy-optimization ablation setting.} For the policy-optimization ablation, we compare the policy-gradient objective used in \method against a GRPO-style update. Both use Qwen2.5-1.5B agents, response-only inputs, contribution-based ordering, weighted aggregation, and XVerify-based evaluation at temperature $0.0$. Evaluation is conducted on AQUA-RAT at $N=5$ for both same-task training, AQUA-RAT$\rightarrow$AQUA-RAT, and cross-task transfer, GSM8K$\rightarrow$AQUA-RAT. 

Both variants use `transformer' backbone, $50$ policy-training iterations, batch size $32$, learning rate $0.1$, a hidden dimension $128$, $1$ transformer layers, $1$ attention heads, dropout $0.3$, Adam optimizer, gradient clipping $1.0$, and batch-mean baseline. For GRPO, we use $4$ rollouts.

\section{Additional Experiments}

\subsection{Extended Task Transfer}

We additionally provide studies on whether a policy trained on one source task can transfer to additional target tasks beyond the main transfer experiments. In this setting, \method is trained on GSM8K and evaluated without retraining on GSM-Hard~\citep{gao23PAL}, HumanEval~\citep{chen2021humaneval}, and MMLU~\citep{hendrycks2021measuring}. We compare against single-call and chain-of-thought baselines using the same target-task evaluation protocol. As shown in Table~\ref{tab:extended-task-transfer}, the GSM8K-trained \method policy improves over both baselines on GSM-Hard and HumanEval and remains comparable to the baselines on MMLU. These results suggest that the learned communication policy can transfer beyond the training task, especially when the target task benefits from structured multi-agent reasoning. 

\begin{table}[h] 
    \centering
    \caption{Extended task-transfer results. \method is trained on GSM8K and evaluated without retraining on GSM-Hard, HumanEval, and MMLU. Accuracy is reported as mean $\pm$ standard deviation.} 
    \label{tab:extended-task-transfer}
    \small
    \begin{tabular}{rr|cc|c}
        \toprule
        \rowcolor{lightgray} \textbf{Trained on} & \textbf{Tested on} & \textbf{Single} & \textbf{CoT} & \textbf{\method} \\ 
                \midrule
            \multirow{3}{*}{GSM8K} & GSM-Hard 
            & $34.40{\pm}0.72$ 
            & $36.07{\pm}0.31$ 
            & $\mathbf{37.13{\pm}1.21}$ \\
              & HumanEval 
            & $50.41{\pm}3.73$ 
            & $45.93{\pm}1.27$ 
            & $\mathbf{51.42{\pm}0.70}$ \\
              & MMLU      
            & $\mathbf{52.07{\pm}1.03}$ 
            & $51.27{\pm}1.50$ 
            & $52.00{\pm}1.06$ \\
        \midrule
        \multicolumn{2}{r}{\textbf{Avg. Acc.}}
            & $45.63$ 
            & $44.42$ 
            & $\mathbf{46.85}$ \\
        \multicolumn{2}{r}{\textbf{Avg. Rating}}
            & $2.00$ 
            & $2.67$ 
            & $\mathbf{1.33}$ \\
        \bottomrule

    \end{tabular}
\end{table}

\end{document}